\documentclass[letterpaper, 10 pt, conference]{ieeeconf} 

\IEEEoverridecommandlockouts

\title{Sensor selection for detecting deviations from a planned itinerary}

\author{Hazhar Rahmani \qquad Dylan A.\ Shell \qquad Jason M.\ O'Kane\thanks{
H. Rahmani and J. M. O'Kane are with the Dept. of Computer Science and Engineering, University of South Carolina, Columbia, SC, USA.   D. A. Shell is with the Dept. of Computer Science and Engineering, Texas A\&M University, College Station, TX, USA.  {\tt\small hrahmani@email.sc.edu, dshell@tamu.edu, jokane@cse.sc.edu} \quad This material is based upon work supported by the NSF under Grants 1849249 \& 1849291.}}

\usepackage{amsmath}
\usepackage{mathtools}
\usepackage{textcomp}
\usepackage{algorithmic}
\usepackage[ruled,vlined,noend,linesnumbered]{algorithm2e}
\usepackage[english]{babel}
\usepackage{amssymb}
\usepackage{graphicx}
\usepackage{xcolor}
\usepackage{enumerate}
\usepackage{framed}
\usepackage{cite}
\usepackage{tabularx}
\usepackage[framemethod=default,innerleftmargin=5pt,innerrightmargin=5pt,skipabove=5pt,topline=false,rightline=false,bottomline=false,linewidth=2.5pt,roundcorner=5pt]{mdframed}

\usepackage{complexity}
\let\oldNP\NP
\renewcommand{\NP}{\oldNP\xspace}

\newcommand{\pagebudget}[1]{}
\newcommand{\showtotalpagebudget}[1]{}

\newtheorem{definitionenv}{\bf Definition}
\newtheorem{lemmaenv}{Lemma}
\newtheorem{theoremenv}{Theorem}
\newtheorem{corollaryenv}{Corollary}
\newtheorem{exampleenv}{Example}

\renewcommand{\emptyset}{\varnothing}

\newenvironment{definition}{\vspace{0.05in}\begin{definitionenv}\em}{\end{definitionenv}\vspace{0.05in}}
\newenvironment{lemma}{\vspace{0.05in}\begin{lemmaenv}\em}{\end{lemmaenv}\vspace{0.05in}}
\newenvironment{theorem}{\vspace{0.05in}\begin{theoremenv}\em}{\end{theoremenv}\vspace{0.05in}}
\newenvironment{corollary}{\vspace{0.05in}\begin{corollaryenv}\em}{\end{corollaryenv}\vspace{0.05in}}

\newcommand*{\probleminternal}[4]{
	\par
	\medskip
	\noindent\fbox{\parbox{0.98\columnwidth}{
			\textbf{#4: #1} \\[0.05in]
			\renewcommand{\tabcolsep}{2pt}
			\begin{tabularx}{\linewidth}{rX}
				\emph{Input:} & #2 \\
				\emph{Output:} & #3
			\end{tabularx}
		}}
		\par
		\medskip
		\par
	}

	\newcommand*{\problem}[3]{\probleminternal{#1}{#2}{#3}{Problem}}
	\newcommand*{\decproblem}[3]{\probleminternal{#1}{#2}{#3}{Decision Problem}}

	\makeatletter
	\newcommand*{\Relbarfill@}{\arrowfill@\Relbar\Relbar\Relbar}
	\newcommand*{\xeq}[2][]{\ext@arrow 0055\Relbarfill@{#1}{#2}}
	\makeatother

	\providecolor{added}{rgb}{0,0,0}
 \providecolor{changed}{rgb}{0,0,0}
 \providecolor{altered}{rgb}{0,0,0}

\newcommand{\altered}[1]{{\color{altered}{}#1}}

\newcommand{\deleted}[1]{}

\newcommand{\walks}{\operatorname{Walks}}
\newcommand{\src}{{\rm src}}
\newcommand{\srcfunc}[1]{\src({#1})}
\newcommand{\tgt}{{\rm tgt}}
\newcommand{\tgtfunc}[1]{\tgt({#1})}
\newcommand{\pow}[1]{\ensuremath{\raisebox{.15\baselineskip}{\Large\ensuremath{\wp}}({#1})}\xspace}

\newcommand{\Yes}{{\rm Yes}\xspace}
\newcommand{\No}{{\rm No}\xspace}
\newcommand{\MSSTVS}{{\rm MSSVI}\xspace}
\newcommand{\MSSTVSDEC}{{\rm MSSVI-DEC}\xspace}

\newcommand*{\gobble}[1]{}
\newcommand*{\gobblexor}[2]{#2}

\begin{document}

\maketitle
		
\begin{abstract} 
Suppose an agent asserts that it will move through an environment in some way.
When the agent executes its motion, how does one verify the claim?  The problem arises
in a range of contexts including validating safety claims about robot
behavior, applications in security and surveillance, and for both the conception and
the (physical) design and logistics of scientific experiments.  Given a set
of feasible sensors to select from, we ask how to choose sensors 
optimally in order to ensure
that the agent's execution does indeed fit its pre-disclosed itinerary.
Our treatment is distinguished from prior work in sensor selection by two
aspects: the form the itinerary takes (a regular language of transitions) and
that families of sensor choices can be grouped as a single choice. Both are intimately tied
together, permitting construction of a product automaton because the same
physical sensors (i.e., the same choice) can appear multiple times.  This paper
establishes the hardness of sensor selection for itinerary validation within this treatment, and
proposes an exact algorithm based on an \altered{integer linear programming (ILP)} formulation that is capable of solving problem
instances of moderate size. We demonstrate its efficacy on small-scale case studies, including
one motivated by wildlife tracking.  
\end{abstract}

\section{Introduction}\pagebudget{1}
Determining how agents within an environment are behaving and understanding
that behavior, for instance by recognizing whether it fits some pattern, is
crucial to the problem of situational awareness, which broadly encompasses
agent detection and tracking, activity modeling, general sense-making, and
semantically-informed surveillance. 
It forms an important capability for
intelligent systems and is a topic of interest for the robotics community for
at least three reasons.  
First, as information consumers: such information
could enhance the ability of a robot to act within context, improving the
responsiveness and appropriateness of robot actions to other events. Secondly,
as information producers: we may wish to task a robot with providing raw sensor
information to enable coverage and facilitate such situational awareness.
The third reason is one shared of technical interest: the methods and
algorithms that enable such situational awareness have substantial overlap with
those used for estimation on-board robots and, historically, cross-pollination
between the two has been fruitful.  
The present paper fits within the vein of work concerned with guarding an
environment~\cite{orourke87art}, though is closer to the minimalist spirit
of~\cite{tovar09beams} both in terms of sensors\,---we adopt a simple model
well suited to information-impoverished sensors such as occupancy and beam sensors---\,but also in the use of combinatorial
filters\,---for fusing sequential observations and estimating state.

\begin{figure}[t]
    \centering
    \includegraphics[width=\linewidth]{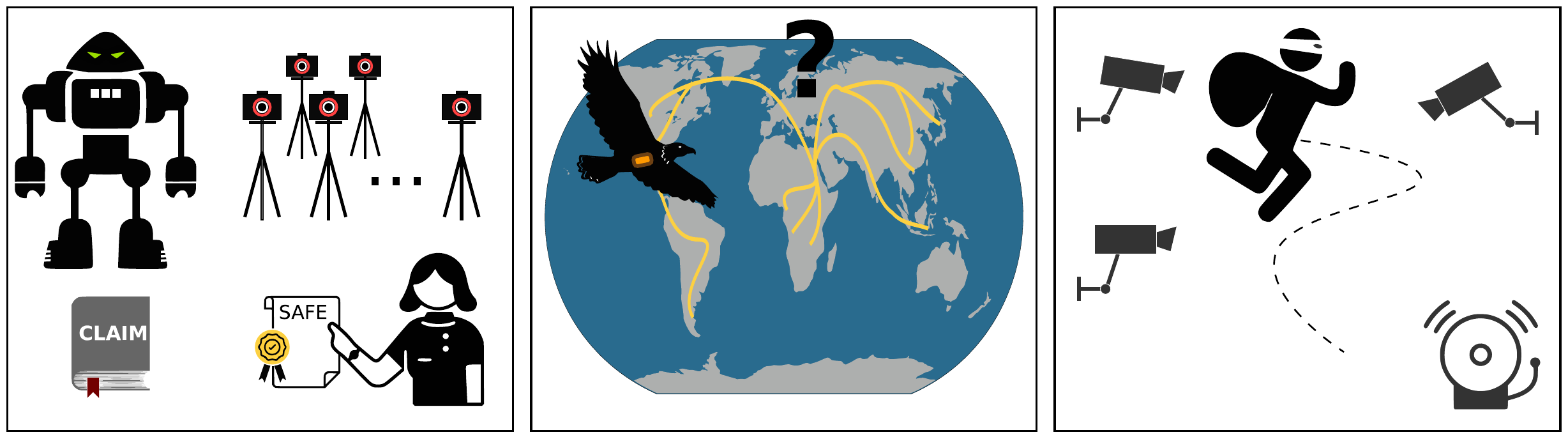}\\
\vspace*{-4pt}
\textbf{a)} \hspace*{2.2cm} \textbf{b)} \hspace*{2.2cm} \textbf{c)}\\
\vspace*{-7pt}
  \caption{
    Three quite distinct settings but which can be treated via the formulation described in this paper.
    \textbf{a)} Wishing to ensure that claims made about a particular robot's behavior will hold,  an engineer chooses to instrument her \gobblexor{laboratory and}{laboratory's} testing environment (in this case, with a set of motion tracking cameras) so she can subject it to adequate scrutiny.
    \textbf{b)} An ornithologist would like to test a new hypothesis concerning the migration patterns of a\gobble{ particular} species of bird, so assesses the cost of acquiring tracking capabilities with adequate power of discernment.
    \textbf{c)}~Rather than purchasing sensors to maximize coverage, \gobblexor{if benign activities in environment}{if benign activities} can be precisely characterized, then negating that description permits surveillance which is sufficient to identify anomalous activities but with far fewer sensors.
    \label{fig:intro}}
    \vspace{-7pt}
\end{figure}

Our work was inspired by Yu and LaValle~\cite{yu2010cyber}, who consider the
question of validating a story: given a polygonal environment, a claimant
provides a sequence of locations which they assert to have visited and the
system is tasked with determining whether a given sequence of sensor readings
is consistent with that claim.  That is, does the sensor history contain any
evidence that the given sequence of locations was in fact not visited?  ~First
in \cite{yu2010cyber}, and then, with several refinements and under weaker
assumptions, in the follow-up \cite{yu2011story}, Yu and LaValle provide an
efficient method for this problem. Their approach is sound and complete in that
it identifies inconsistencies between the story and the sensed history if and
only if such inconsistencies exist.  However, the strength of the validation
(or, more correctly, the method's \textsl{in}ability to \textsl{in}validate)
must be understood modulo sensor data. The faculty to detect contradictions
depends critically on sensor history, on the evidence that the sensors provide.
The more limited the sensing, the fewer fibs you can catch.

A natural concern, then, is how to choose sensors.  Suppose that the given
story describes a pre-declared itinerary, a future path or structured collection of possible paths through an environment.
Now, given a set of possible sensors one could deploy, when some path has been
executed, which ones suffice to detect deviations from the itinerary?  The
present paper considers an optimization variant where we ask for a minimal set of sensors that can accomplish this.
Fig.~\ref{fig:intro} illustrates the breadth of use cases for this
scenario, in which the goal is to ensure detection of all deviations. It is
important to note that with a given environment, given itinerary, and set of
sensors, this may be impossible---the whole set of sensors may be
inadequate.  A concrete sort of application, based on selection of a suite of
beam and occupancy sensors in an indoor environment, appears in~Fig.~\ref{fig:department}. 

This paper starts by formalizing the problem of optimal sensor selection to
detect deviations from a disclosed itinerary (Section~\ref{sec:defn}).  As
itineraries are claims about the future, it seems useful to permit rather more
flexibility than the stories of Yu and LaValle allow, and so our treatment
does.  Further, motivated by the physical placement of sensors which
effectively guard multiple areas at once, it models circumstances where
multiple sensors may be obtained together as a single logical unit. 
%
We then establish the computational hardness of the problem (in
Section~\ref{sec:hardness}), and show how it may be treated via integer linear programming (in Section~\ref{sec:ilp}). Experimental results in Section~\ref{sec:caseStudies} show 
that realistic size instances can be practically solved in this way.

\section{Related Work}\pagebudget{0.3}

The problem of reducing the sensor readings needed to establish some property
has been the subject of extensive study in the discrete event systems
literature (see the survey\altered{ by Sears and Rudie~\cite{sears16minimal}}).  That literature
distinguishes sensor selection (e.g.,~\cite{haji96minimizing}) from sensor
activation (e.g.,~\cite{cassez2008fault,wang2018optimizing}). The former considers, as studied in
the present paper, a one-shot decision at initialization of whether to adopt a sensor or
not; the latter is an online variant that switches sensors on/off across time.
This paper aims to fill the niche between such sensor-oriented
work and the problem of story validation (as exemplified by \altered{Yu and LaValle~\cite{yu2010cyber,yu2011story}}).

The $\NP$-completeness of minimal sensor selection for the properties of observability and
diagnosability~\cite{sampath1995diagnosability} was established
\altered{by Yoo and Lafortune~\cite{yoo02NP}}.  Recently, Yin and Lafortune~\cite{yin2019general} proposed
a general approach to optimizing sensor selection that applies to a very wide
set of problems and subsumes several previous methods. Their approach is
capable of enforcing what they term `information state-based properties,'
essentially arbitrary predicates defined on states, which allow one treat a
variety of fault detection and diagnosis tasks (including observability and
diagnosability).  Our itinerary validation problem, however, doesn't fall within
this class as it is not enough to ask whether a state is visited or not;
specific transitions between states matter too.  

\begin{figure}[t]
  \centering
  \includegraphics[width=0.47\linewidth, trim={0 1.5cm 10.7cm 0}, clip]{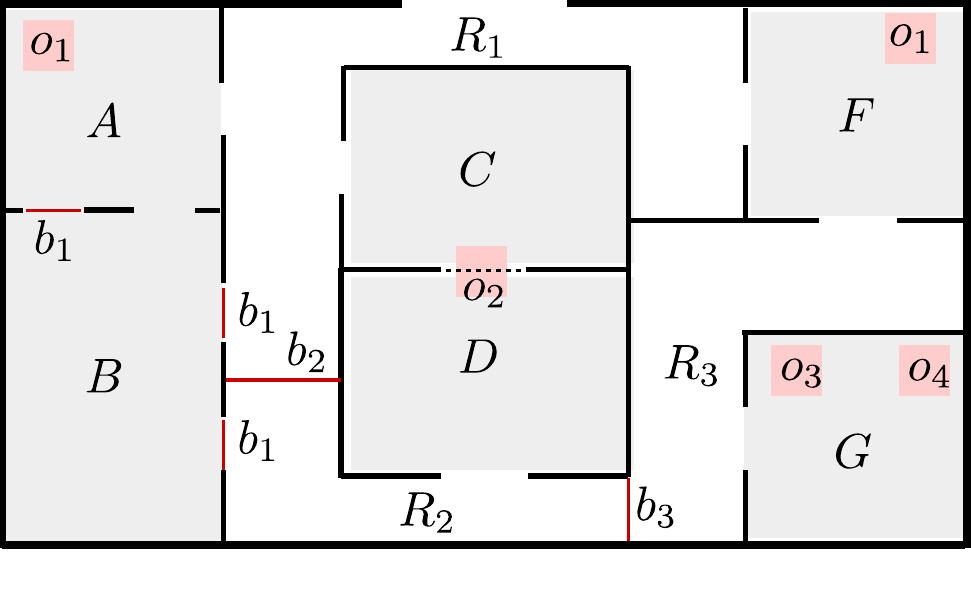}
  \includegraphics[width=0.51\linewidth, trim={10.3cm 1.3cm 0 0}, clip]{pics/department.pdf}
  \phantom{.}\\
  \vspace{-15pt}
\textbf{~a)} \hspace*{4.0cm} \textbf{b)} \\
\vspace*{-7pt}
  \caption{
    \textbf{a)} An example of an environment, which shows the floor map of a department. 
    This environment is guarded by beam sensors $b_1$, $b_2$, and $b_3$, and by occupancy sensors $o_1$, $o_2$, $o_3$, and $o_4$.
    \gobblexor{\textbf{b)} A world graph representing the environment. Note that the graph is a multigraph. Each edge $e$ is labeled by a world-observation---a set of events---the system receives when an agent moves from the region represented by the source vertex of $e$ to the region represented by the target vertex of~$e$.}{\textbf{b)} A multigraph representing the environment. Each edge $e$ is labeled by a world-observation, a set of events received when an agent moves from the region represented by the source vertex of $e$ to the region represented by the target vertex of~$e$.}
  \label{fig:department}}
\vspace*{-18pt}
\end{figure}

The characteristic that differentiates itinerary validation is that the basic
mathematical objects under consideration are trajectories.  This emphasizes an
important distinction between language- and automata-based properties
(cf. our~\cite{rahmani2020planning}), with fairly subtle implications for our
model and approach. Itineraries will be taken to describe sequences of edges and in order
to relate the world's structure to those sequences, we will form a product.
That product may partially unroll or unfold the world, so that the choice of a sensor can
affect elements less `locally' than one might naturally expect. One
implication for the model is that we directly treat situations where multiple
sensors may be selected as a single logical unit as this
occurs in the product anyway. Thence, the further implication is that our
computational complexity (hardness) result utilizes a cover selection problem
directly.

\section{Definitions and Problem Statement}\pagebudget{1.7}
\label{sec:defn}

This section formalizes our sensor selection problem.

\vspace*{-8pt}

\subsection{Modeling the environment}

        In our approach, the environment is modeled using a discrete structure, called a world graph, defined as follows.
        
        \begin{definition}[World graph]\label{def:wg}
		A \emph{world graph} is an edge-labeled directed multigraph $\mathcal{G} = (V, E, \src, \tgt, v_0, S,\mathbb{Y},\lambda)$ in which 
		\begin{itemize}
		    \item $V$ is a nonempty vertex set, 
		    \item $E$ is a set of edges,
		    \item $\src: E \rightarrow V$ and $\tgt: E \rightarrow V$ are \emph{source} and \emph{target} functions, respectively, which identify the source vertex and target vertex of each edge,
		    \item $v_0 \in V$ is an initial vertex, 
		    \item $S = \{s_1, s_2, \dots, s_k\}$ is a nonempty finite set of sensors, 
		    \item $\mathbb{Y} = \{Y_{s_1}, Y_{s_2}, \dots, Y_{s_k}\}$ is a collection of mutually disjoint event sets associated to each sensor,  and
		    \item $\lambda: E \rightarrow  \pow{Y_{s_1}\cup Y_{s_2}\cup \dots \cup Y_{s_k}} $ is a labeling function, which assigns to each edge, a \emph{world-observation}---a set of events.
		\end{itemize}
		\end{definition}
	    
        (Here $\pow{X}$ denotes the set of all subsets of $X$.)
        \smallskip

		The intuition is that a world graph describes an environment through which agent might move, along with the available sensors in the environment and also the sensor events that agent may trigger along its motion, provided those sensors are active.
		Vertices in the graph represent regions within the environment of interest.
		Edges represent feasible transitions between regions, each labeled with a set of sensor events that happen simultaneously when the system makes the transition corresponding to that edge.

        In this model, for each sensor $s_i \in S$, $Y_{s_i}$ is the set of all events produced by $s_i$. Notice that Definition~\ref{def:wg} stipulates that distinct sensors produce disjoint events, that is, for each $s_i, s_j \in S$, if $s_i \neq s_j$, then $Y_{s_i} \cap Y_{s_j} = \emptyset$.
        The labeling function $\lambda$ is used to indicate for each edge, the set of events produced by the sensors when the agent makes a transition corresponding to that edge in the environment.

\subsection{Example: Beam and occupancy sensors}        
Though the technical results to follow apply for any world graph that satisfies Definition~\ref{def:wg}, to keep the description reasonably concrete, we present examples focused upon occupancy sensors (which detect the presence of the agent in a region) and beam sensors (which detect the passage of an agent between adjacent regions).

To illustrate, consider the simple environment in Fig.~\ref{fig:department}a,
    which is guarded by four occupancy sensors $o_1$, $o_2$, $o_3$, and $o_4$ and three beam sensors $b_1$, $b_2$, and $b_3$.
%
%
Fig.~\ref{fig:department}b shows the world graph corresponding to this environment, as constructed via the algorithm of Yu and LaValle~\cite{yu2010cyber}.
Each state of this graph represents a room or a region within the environment guarded by the same set of sensors.
Each edge shows a transition between two neighboring regions. 
%

Notice that, in the example, some pairs of rooms are connected by multiple doors, each of which are guarded by different sensors. This explains why Definition~\ref{def:wg} uses a multigraph structure for world graphs.
Also note that in this environment, an agent cannot directly move between rooms $C$ and $D$ because those two rooms are separated by a window rather than a door.  

In this example, an occupancy sensor $o$, which detects the presence of an agent in a region $X$, is activated when the agent enters $X$ and is deactivated once the agent \textcolor{black}{exits} $X$.
Accordingly, each occupancy sensor $o$ in this example produces two events, $o^+$, which occurs when $o$ is activated, and $o^-$, which happens when $o$ is deactivated. 
\textcolor{black}{A beam sensor is activated when a mobile agent physically crosses (or breaks) the beam and then instantly deactivated.}
Because the agent is mobile and a beam sensor is deactivated immediately after it was activated, 
we model each beam sensor $b$ with a single event $b$ in $Y_{b}$.
\textcolor{black}{
When beam $b$ is broken, the system knows that the physical line segment was crossed, but not the direction of crossing.}
For the environment in Fig.~\ref{fig:department}:
%
%
for each of the occupancy sensors $o_i$, $i \in \{1, 2, 3, 4\}$, $Y_{o_i}=\{o_i^+, o_i^- \}$; for each of the beam sensors $b_j$, $j \in \{1, 2, 3 \}$, $Y_{b_j}=\{b_j\}$.

When an agent makes a transition between two regions, it is possible that several events happen simultaneously, and in fact, the system observes all those events at the same time.
For example, when an agent makes a transition from room $A$ to room $B$ from the left door, two events $o_1^-$ and $b_1$ happen simultaneously, and thus, edge $e_3$ in the world graph is labeled with the world-observation $\{b_1, o_1^- \}$.

Finally, notice that a world graph can readily represent scenarios in which a single sensor guards multiple transitions.
In the example, $B$ has four doors, three of which are guarded. The beams that guard those doors are assumed to be a single beam sensor $b_1 \in S$. 
When an agent crosses any of those beams, the system knows that one of them is crossed but it does not know which one it was.
Likewise, 
%
rooms $C$ and $D$ are guarded by a single occupancy sensor $o_2$, which is located on the window between those two rooms.
Thus, if an agent enters any of those two rooms, $o_2$ is activated, but by observing $o_2^+$, the system cannot tell if the agent entered room $C$ or room $D$.
%
        

%
\subsection{Itinerary DFA}
In our story validation problem, an agent takes a tour in the environment along a continuous path. This path
is represented over the world graph $\mathcal{G}$ by a \emph{walk}, which is defined as a finite sequence of edges $e_1 e_2 \cdots e_n \in E^*$ in which
$\srcfunc{e_1}=v_0$ and for each $i \in \{ 1, 2, \ldots, n-1\}$, $\tgtfunc{e_i}=\srcfunc{e_{i+1}}$.
The set of all walks over $\mathcal{G}$ ---that is, 
the set of all not-necessarily-simple paths one can take in the environment--- is denoted $\walks(\mathcal{G})$.

The agent claims that its tour will be one of those words specified by a deterministic finite automaton (DFA):
\begin{definition}[Itinerary DFA]
    An \emph{itinerary DFA} over a world graph $\mathcal{G} = (V, E, \src, \tgt, v_0, S,\mathbb{Y},\lambda)$ is a DFA $\mathcal{I} = (Q, E, \delta, q_0, F)$ in which $Q$ is a finite set of states; $E$ is the alphabet; $\delta: Q \times E \rightarrow Q$ is the transition function; $q_0$ is the initial state; and $F$ is the set of all accepting (final) states.
\end{definition}
    
For each finite word $r = e_1 e_2 \cdots e_n \in E^*$, there is a unique sequence of states $q_0 q_1 \cdots q_n$ for which $q_0$ is the initial state and for each $i \in \{0, 1, \ldots, n-1 \}$, $\delta(q_i, e_i) = q_{i+1}$.
Word $r$ is \emph{accepted} by the DFA if $q_n \in F$.
The language of $\mathcal{I}$, denoted $L(\mathcal{I})$, is the set of all finite words accepted by $\mathcal{I}$, i.e., $L(\mathcal{I})=\{ r \in E^* \mid r \ \text{is accepted by } \mathcal{I} \}$.
The robot claims its tour will be one of the words $r \in L(\mathcal{I})$.
Note that each word accepted by this DFA is a walk over the world graph,
and accordingly, the robot's itinerary not only specifies the sequence of locations the agent visits but it also identifies the specific transitions (i.e. the doors between the rooms) through which the agent \textcolor{black}{moves}.

\subsection{Itinerary validation}
%
%

%
We seek to enable a minimal subset $M \subseteq S$ of sensors such that, when the agent finishes its tour within the environment, the system can determine with full certainty whether the agent followed its itinerary or not. 
At the completion of the agent's tour, the system does not know the exact tour the agent took in the environment, instead receiving only a sequence of world-observations.
\textcolor{black}{Each item in the sequence is generated by the system when a set of sensors were activated or deactivated simultaneously as a result of the agent's moving in the environment.}
%

%
%

Let us mildly abuse notation and use $Y_{M}$ to denote the set of all events
produced by sensors in some $M \subseteq S$, i.e., $Y_{M} = \bigcup_{s \in M} Y_s$.
If, from all sensors $S$, only the sensors in $M$ are turned on, then when the agent transitions across $e$ in $\mathcal{G}$, the system receives world-observation $\lambda(e) \cap Y_{M}$. That is, when transitioning across an edge, the system observes precisely those events that are both associated with that edge and enabled by one of set $M$'s selected sensors.
Where $\lambda(e) \cap Y_{M} = \emptyset$, this must be handled slightly differently: in this case, the system produces no symbol at all (not a symbol reporting that some un-sensed event occurred---which would itself be a tacit sort of information). We make this precise next.

The agent's walk over the world graph generates a sequence of non-empty world-observations.
For a world graph  $\mathcal{G} = (V, E, \src, \tgt, v_0, S,\mathbb{Y},\lambda)$, we define function $\beta_{\mathcal{G}}: \walks(\mathcal{G}) \times \pow{S}  \rightarrow (\pow{Y_S} \setminus \emptyset)^*$ in which
    for each $r \in \walks(\mathcal{G})$ and subset of sensors \mbox{$M \subseteq S$}, $\beta_{\mathcal{G}}(r, M)$ gives the sequence of world-observations the system receives when the agent takes walk $r$ and precisely the sensors in $M$ are turned on.
    Formally, for each $r  =  e_1 e_2 \cdots e_n \in \walks(\mathcal{G})$, $\beta_{\mathcal{G}}(r, M) = z_1 z_2 \cdots z_{n}$ in which for each $i \in \{1, \ldots, n \}$, $z_i  = \lambda(e_i) \cap Y_M$ if $(\lambda(e_i) \cap Y_M) \neq \emptyset$, and $z_i = \epsilon$ otherwise, where $\epsilon$ is the (standard) \emph{empty symbol}.


Based on this function, we make a definition that formulates conditions under which a set of sensors are able to tell whether the agent adhered to its claimed itinerary or not.
\begin{definition}[Certifying Sensor Selection]
    \label{def:legSens}
    Let $M \subseteq S$ be a subset of sensors. We say $M$ \emph{certifies} itinerary $\mathcal{I}$ on world graph $\mathcal{G}$ if there exist no $r \in L(\mathcal{I}) \cap \walks(\mathcal{G})$ and $t \in \walks(\mathcal{G}) \setminus L(\mathcal{I})$ such that $\beta_{\mathcal{G}}(r, M) = \beta_{\mathcal{G}}(t, M)$.
\end{definition}

Intuitively, if $M$ is a certifying sensor selection for $\mathcal{I}$, then based on the sequence of world-observations $\beta_{\mathcal{G}}(r, M)$ the system perceives from the environment, the system can tell whether $r$ was within the claimed itinerary or not, that is, whether $r \in L(\mathcal{I})$ or not.
In fact, if for each $r \in L(\mathcal{I}) \cap \walks(\mathcal{G})$, there is no $t \in (\walks(\mathcal{G}) \setminus L(\mathcal{I}))$ such that $\beta_{\mathcal{G}}(t, M) = \beta_{\mathcal{G}}(r, M)$, then
the system can tell with full certainty if $r$ was within the claimed itinerary or not.
%
Thus, the system must choose a sensor set $M$ to turn on that is certifying for $\mathcal{I}$ on $\mathcal{G}$.
%

We formalize our minimization problem as follows.
        \problem{Minimal sensor selection to validate an itinerary (\MSSTVS)}
{A world graph $\mathcal{G} = (V, E, \src, \tgt, v_0, S,\mathbb{Y},\lambda)$ and an itinerary DFA $\mathcal{I} = (Q, V, \delta, q, F)$.}
{A minimum size certifying sensor selection \mbox{$M \subseteq S$} for $\mathcal{I}$ on $\mathcal{G}$, or `\textsc{Infeasible}' if no such certifying sensor selection exists.}
\gobblexor{We show that this problem is $\NP$-hard, and present an algorithm to solve it via integer linear programming.
But first, we present a construction for a certain type of product automaton which will be useful for those primary results.}{
Before showing this problem to be $\NP$-hard, we describe a construction that turns out to be useful in what follows.}

		\section{World graph-itinerary product automata}\pagebudget{0.5}
In this section, we describe how to use the inputs of the \MSSTVS problem to construct a product automaton that captures the interactions between a world graph and an itinerary DFA.  We use this construction for both a hardness result about \MSSTVS (in Section~\ref{sec:hardness}) and a practical solution of \MSSTVS via integer linear programming (in Section~\ref{sec:ilp}).
         This product automaton is defined as follows.
         \begin{definition}[Product automaton]
         \label{def:prodAut}
           Let  $\mathcal{G} = (V, E, \src, \tgt, v_0, S,\mathbb{Y},\lambda)$ be a world graph and $\mathcal{I} = (Q, E, \delta, q_0, F)$ be an itinerary DFA. The product automaton $\mathcal{P}_{\mathcal{G}, \mathcal{I}}$ is 
           a partial DFA $\mathcal{P}_{\mathcal{G}, \mathcal{I}} = (Q_\mathcal{P}, E, \delta_{\mathcal{P}}, q_0^{\mathcal{P}}, F_{\mathcal{P}})$ with
           \begin{itemize}
               \item $Q_{\mathcal{P}} = Q \times V$,
               \item $\delta_{\mathcal{P}}: Q_{\mathcal{P}} \times E \nrightarrow Q_{\mathcal{P}}$ is a partial function such that for each $(q, v) \in Q_{\mathcal{P}}$ and $e \in E$, $\delta_{\mathcal{P}}((q, v), e)$ is undefined if 
               $\srcfunc{e} \neq v$, otherwise, $ \delta_{\mathcal{P}}((q, v), e) = (\delta(q, e), \tgtfunc{e})$,
               \item $q_0^{\mathcal{P}}=(q_0, v_0)$, and
               \item $F_{\mathcal{P}}=F \times V$.
           \end{itemize}
         \end{definition}
         Note that the transition function of this DFA is partial. We will write $\delta_{\mathcal{P}}(p, e)=\bot$ to mean that $\delta_{\mathcal{P}}$ is undefined for $(p, e)$.
        The extended transition function $\delta_{\mathcal{P}}^*: Q_{\mathcal{P}} \times E^* \nrightarrow Q_{\mathcal{P}}$---which for each $q \in Q_{\mathcal{P}}$ and $r \in E^*$, $\delta_{\mathcal{P}}^*(q, r)$ denotes the state to which the DFA reaches by tracing $r$ from state $q$---is also partial.
        For a word $r=e_1 e_2 \cdots e_n \in E^*$, we use $\delta_{\mathcal{P}}^*(q_0^{\mathcal{P}}, r) = \bot$ to mean that this DFA crashes when it traces $r$ from the initial state, that is, there is a unique state sequence $q_0 q_1 \cdots q_k$ for some $k < n$  such that $\delta_{\mathcal{P}}(q_{i-1}, e_{i})=q_{i}$ for all $i \in \{1, 2, \ldots, k-1 \}$ but $\delta_{\mathcal{P}}(q_k, e_k)=\bot$.
        A word $r \in E^*$ is \emph{trackable} by this DFA if $\delta_{\mathcal{P}}^*(q_0^{\mathcal{P}}, r) \neq \bot$. 
        
         Our purpose in constructing this product automaton is revealed by the following result. 
         %
         %
         %
         \begin{lemma}
         \label{lem:prodAut}
         Let $\mathcal{G}$, $\mathcal{I}$, and $\mathcal{P}_{\mathcal{G}, \mathcal{I}}$ be the structures in Definition~\ref{def:prodAut}. 
         A subset $M \subseteq S$ of sensors is a certifying sensor selection for $\mathcal{I}$ if and only if for each $r, r' \in E^*$ such that $\delta_{\mathcal{P}}^*(q_0^{\mathcal{P}}, r) \in F_{\mathcal{P}}$ and 
         $\delta_{\mathcal{P}}^*(q_0^{\mathcal{P}}, r') \in Q_{\mathcal{P}} \setminus F_{\mathcal{P}}$, it holds that $\beta_{\mathcal{G}}(r, M) \neq \beta_{\mathcal{G}}(r', M)$.
         \end{lemma}

         \begin{proof}
                \gobblexor{From the construction of $\mathcal{P}$, two observations may readily be made.}{The construction yields two direct observations:}
                \begin{enumerate}
                    \item[(1)] Every word trackable by $\mathcal{P}$ is a walk over $\mathcal{G}$ and vice versa, i.e., $\{r \in E^* \mid \delta_{\mathcal{P}}^*(q_0^{\mathcal{P}}, r) \neq \bot \} = \walks(\mathcal{G})$.
                    \item[(2)] The accepting states of $\mathcal{P}$ correspond to the accepting states of the itinerary DFA, so $L(\mathcal{P}_{\mathcal{G}, \mathcal{I}}) = L(\mathcal{I})$.
                 \end{enumerate}
                 Taken together, (1) and (2) imply that 
                 $\walks(\mathcal{G}) \setminus L(\mathcal{I}) = \{r \in E^* \mid \delta_{\mathcal{P}}^*(q_0^{\mathcal{P}}, r) \in Q_{\mathcal{P}} \setminus F_{\mathcal{P}} \}$.
                 This means that each walk over the world graph that is not within the language of the itinerary DFA, reaches a non-accepting state in this DFA.
                 But, because every word in the the language of the itinerary DFA reaches to an accepting state in this DFA,
                 if there are two words $r$ and $r'$ such that $r$ reaches to an accepting state, $r'$ reaches to a non-accepting state, and $r$ and $r'$ both yield the same sequence of
                 world-observations by $\beta_{\mathcal{G}}$ under $M$, i.e., $\beta_{\mathcal{G}}(r, M)=\beta_{\mathcal{G}}(r', M)$, then $M$ is not a certifying sensor selection for $\mathcal{I}$.  Contrariwise, if there are no such words $r$ and $r'$, then $M$ is a certifying sensor selection.
                 This completes the proof.
         \end{proof}

        As a result, given a sensor selection $M$ as a feasible solution to \MSSTVS with inputs $\mathcal{G}$ and $\mathcal{I}$, one can use the product automaton $\mathcal{P}_{\mathcal{G}, \mathcal{I}}$ to check  
        if $M$ is a certifying sensor selection for $\mathcal{I}$ or not by testing whether such $r$ and $r'$ described in the proof of this lemma can be found or not.
        The next section makes this idea clear.
        %
        

        \section{Hardness of \MSSTVS}\label{sec:hardness} \pagebudget{1.5}
        Next, we present a hardness result for minimal sensor selection, starting by casting \MSSTVS as a decision problem.
        \decproblem{Minimal sensor selection to validate an itinerary (\MSSTVSDEC)}
		{A world graph $\mathcal{G} = (V, E, \src, \tgt, v_0, S,\mathbb{Y},\lambda)$, an itinerary DFA $\mathcal{I} = (Q, V, \delta, q_0, F)$, and integer~$k$.}
		{$\Yes$ if there is a certifying sensor selection $M \subseteq S$ such that $|M| \leq k$; $\No$ otherwise.
		}
		We prove that \MSSTVSDEC is \NP-complete, by showing that it is both in \NP and \NP-hard.  First, we show that \MSSTVSDEC can be verified in polynomial time.
		\begin{lemma}
		\label{lem:NP}
        \MSSTVSDEC $\in \NP$.
		\end{lemma}
		\begin{proof}
		        We need to show that, using a given sensor selection $M \subseteq S$ as a certificate, we can verify in polynomial time both (1) whether $|M| \leq k$ and (2) whether $M$ is a certifying sensor selection
		        in the sense of Definition~\ref{def:legSens}
		        or not.
		        As (1) is trivially verifiable, we turn to (2).

		        Recall that by Lemma~\ref{lem:prodAut}, if for any words $r, r' \in E^*$ for which $\delta_{\mathcal{P}}^*(v_0, r) \in F_{\mathcal{P}}$ and $\delta_{\mathcal{P}}^*(v_0, r') \in Q_{\mathcal{P}} \setminus F_{\mathcal{P}}$, it holds that $\beta_{\mathcal{G}}(r, M) \neq \beta_{\mathcal{G}}(r', M)$, then $M$ is a certifying sensor selection for the given itinerary DFA $\mathcal{I}$, otherwise $M$ is not a certifying sensor selection for $\mathcal{I}$.
		        Thus, to check if $M$ is certifying, we compute, using a fixed point algorithm described presently, a relation $R$ that relates all pairs of states $q$ and $p$ that are reachable from the initial state by two words (walks) that yield the same sequence of world-observations by $\beta_{\mathcal{G}}$ under $M$.  Then we check whether $R$ relates any pairs of states $q$ and $p$ such that one of them is accepting while the other is non-accepting. 
		        If $R$ relates such a pair, then $M$ is not a certifying sensor selection for $\mathcal{I}$. Otherwise, $M$ is certifying for $\mathcal{I}$.
		        %
		        %
		        %
		        To construct $R$, begin with $R$ initially assigned to $\{(q_0^{\mathcal{P}}, q_0^{\mathcal{P}}) \}$.  Then iteratively update $R$ according to the 
		        following equation until the iteration reaches a fixed point, with no additional tuples added to $R$.
		         \begin{equation*}
		         \begin{split}
		             R \leftarrow R \cup \bigcup_{(q, p) \in R}  \left( 
                     \raisebox{4ex}{$
                     \bigcup\limits_{\substack{e, d \in E: \\ (\lambda(e)\cap Y_M) = \\ (\lambda(d)\cap Y_M) \\ \text{ and } \delta_{\mathcal{P}}(q, e) \neq \bot \\
		             \text{ and } \delta_{\mathcal{P}}(p, d) \neq \bot}} 
                     (\delta_{\mathcal{P}}(q, e), \delta_{\mathcal{P}}(p, d))
                     $}
                     \right) \\
		             \cup \bigcup_{(q, p) \in R} \left( 
                     \raisebox{2.25ex}{$
                     \bigcup\limits_{\substack{e \in E: \\ (\lambda(e)\cap Y_M) = \epsilon \\ \text{ and } \delta_{\mathcal{P}}(q, e) \neq \bot }} (\delta_{\mathcal{P}}(q, e), p)
                     $} 
                     \right)\!\!.
		         \end{split}
		         \end{equation*}
		       To expand $R$, this equation uses two rules, one in the first line and the other in the second line.
		       Fig.~\ref{fig:mp_edge_constraints} illustrates those two rules.
		       The first rule states that if $(q, p) \in R$, then for any edges $e, d \in E$ such that $q$ has an outgoing transition for $e$ and $p$ has an outgoing transition for $d$, if $e$ and $d$ yield the same world-observation under $M$, then we must add $(\delta_{\mathcal{P}}(q), \delta_{\mathcal{P}}(p))$ to $R$ as well, which means that 
		       states $\delta_{\mathcal{P}}(q, e)$ and $\delta_{\mathcal{P}}(p, d)$ are reachable from the initial state of $\mathcal{P}_{\mathcal{G}, \mathcal{I}}$ by two words (walks) that yield the same world observation by $\beta_{\mathcal{G}}$ under $M$.
		       The second rule enforces that if  $(q, p) \in R$, then for any edge $e \in E$ such that $q$ has an outgoing transition for $e$, if $e$ yields the empty world-observation under $M$ (that is, if $e$ yields $\epsilon$ by $\beta_{\mathcal{G}}$ under $M$), then $(\delta_{\mathcal{P}}(q, e), p)$ must be added to $R$ too. This is because here both states $\delta_{\mathcal{P}}(q, e)$ and $p$ are reachable from the initial state, respectively, by a pair of words (walks) $r$ and $r'$ that, respectively, reached $q$ and $p$, while yielding a single sequence of world-observations.
               Using an appropriate implementation, this algorithm takes time which is polynomial in the size of $\mathcal{P}$ because there are $\mathcal{O}(|Q_{\mathcal{P}}|^2)$ pairs in $R$ and thus $\mathcal{O}(|Q_{\mathcal{P}}|^2)$ stages of updates, each checking at most $|E|$ edges. 
		       This shows that \MSSTVSDEC $ \in \NP$.
		\end{proof}

        The practical import of this lemma is that, in polynomial time, we can decide whether a sensor set is certifying for a given itinerary or not.
\begin{figure}[t]
  \centering
  \includegraphics[width=\linewidth]{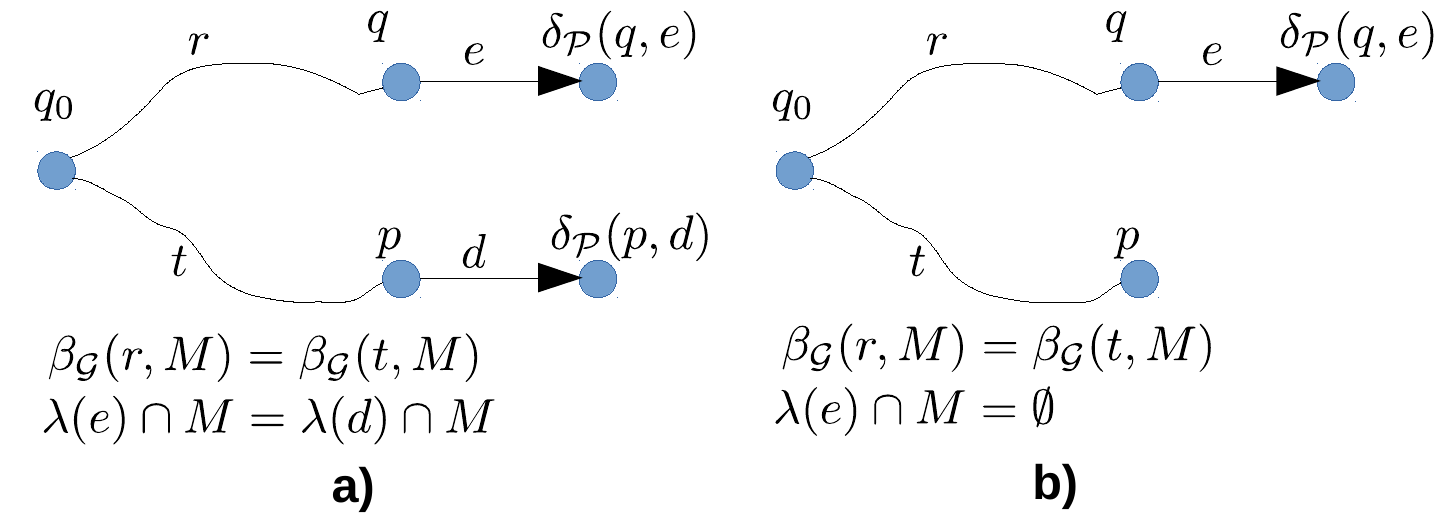}
  \vspace*{-16pt}
  \caption{
    Two cases in the construction of the relation $R$ in the proof of Lemma~\ref{lem:NP}. In each case, $R$ is expanded, given that states $q$ and $p$ are already related by $R$.
    %
    \textbf{a)} In this case, pair $\delta_{\mathcal{P}}(q, e)$ and $\delta_{\mathcal{P}}(p, d)$ are added to $R$.
    \textbf{b)} In this case, pair $\delta_{\mathcal{P}}(q, e)$ and $p$ are added to $R$.
    }
  \label{fig:mp_edge_constraints}
  \vspace*{-12pt}
\end{figure}
        
        Next, we prove that \MSSTVSDEC is computationally hard.
        To do so, we shall reduce from a well-known problem.

        \decproblem{Set cover (SETCOVER-DEC)}
		{A finite set $U$, called the universe, a collection of subsets $\mathbf{O} = \{ O_1, O_2, \ldots, O_m \}$ in which, for each $i \in \{1, 2, \ldots, m \}$, $O_i \subseteq U$ and  $\bigcup_{i \in \{1, 2, \ldots, m \}} O_i = U$, and an integer $k$.}
		{$\Yes$ if there is a sub-collection $\mathbf{N} \subseteq \mathbf{O}$ such that 
		$\bigcup_{O \in \mathbf{N}} O = U$ and 
		$|\mathbf{N}| \leq k$, and $\No$ otherwise.
		}
        %
        %
        This problem is known to be $\NP$-complete~\cite{karp1972reducibility}.
        %
        We reduce from SETCOVER-DEC to prove the following result.
        \begin{theorem}\label{thm:nphard}
 \MSSTVSDEC  is $\NP$-hard.
        \end{theorem}
        
        \begin{proof}
            We prove this result by a polynomial reduction from SETCOVER-DEC to \MSSTVSDEC.
            %
            %
            %
            Given a SETCOVER-DEC instance $$x = \langle U=\{u_1, u_2, \ldots, u_n\}, \textbf{O}=\{O_1, O_2, \ldots, O_m\}, k \rangle,$$ construct an \MSSTVSDEC instance $$f(x) = \langle \mathcal{G} = (V, E, \src, \tgt, v_0, S,\mathbb{Y},\lambda), \mathcal{I}, k' \rangle, $$ 
            as illustrated in Fig.~\ref{fig:reduction} and detailed below.
                \begin{list}{--}{
                    \setlength{\leftmargin}{7pt}
                    \setlength{\itemindent}{0pt}
                    \setlength{\labelsep}{2pt}
                    \setlength{\labelwidth}{7pt}
                    \setlength{\itemsep}{2pt}
                } 
                    \item For the vertices of the world graph, create $2n+2$ states, denoted $V = \{C_0, C_1, \ldots, C_{n+1} \} \cup \{u_1, u_2, \ldots, u_n \}$.  Note that the $u_i$ elements correspond directly to the elements of the universe in the SETCOVER-DEC instance. The idea is that the $u_i$'s represent rooms arranged in sequence, each accessible from a shared corridor composed of the $C_i$'s.
                    
                    \item For the edges of the world graph, create $4n+2$ edges, denoted $E = \{e_0, e_1, \ldots, e_n \} \cup \{e_0', e_1', \ldots, e_n'\} \cup \{d_1, d_2, \ldots, d_n \} \cup \{d_1', d_2', \ldots, d_n' \}$. The $\src$ and $\tgt$ functions are defined so that each $e_i$ connects $C_i$ to $C_{i+1}$, 
                    each $e_i'$ connects $C_{i+1}$ to $C_i$, each $d_i$ connects $C_i$ to $u_i$, and each $d_i'$ connects $u_i$ to $C_i$.
                    
                    
                    \item Create a set of $n+m+1$ sensors, $S = \{b_1, b_2, \ldots, b_{n+1} \} \cup \{O_1, O_2, \ldots, O_m \}$, in which the $b_i$'s are beam sensors and the $O_i$'s are occupancy sensors.
                    The event set corresponds to these sensors in the usual way, with one event for each beam sensor and two events for each occupancy sensor, so
                    for each $j \in \{1, 2, \ldots, n+1 \}$, $Y_{b_j} = \{ b_j\}$, and for each $i \in \{1, 2, \ldots, m\}$, $Y_{O_i} = \{ O_i^+, O_i^-\}$, and then, 
                    $\mathbb{Y} = \{Y_{b_0}, Y_{b_1}, \dots, Y_{d_{n+1}}, Y_{O_1}, Y_{O_2}, \dots, Y_{O_m}\}$.

                    \item For the events labeling each edge, define  $\lambda(e_i)=\lambda(e_i')=b_{i+1}$ for the $e_i$ and $e_i'$ edges, $\lambda(d_i)=\{O_j^+ \mid u_i \in O_j \}$ for the $d_i$ edges, $\lambda(d_i')=\{O_j^- \mid u_i \in O_j \}$ for the $d_i'$ edges.  This models one or more occupancy sensors in each of the $u_i$ rooms, according to the subsets within the SETCOVER-DEC instance, and beam sensors along the corridor between each room.
                    
                    \item For the itinerary DFA $\mathcal{I}$, construct a DFA accepting the singleton language
                    $L(\mathcal{I})=\{e_0 d_1 d_1' e_1 d_2 d_2' e_2 \ldots d_n d_n' e_n\}$ as a linear chain of states.
                    
                    \item For the bound on the number of sensors allowed, choose $k' = k + n + 1$.
                \end{list}        
                In the \MSSTVSDEC instance constructed in this way, notice that, for each subset $O \in \mathbf{O}$, the construction makes a corresponding occupancy sensor $O$ and puts that sensor in all rooms $u$ for which $u \in O$.
                Moreover, the itinerary specifies a single walk $e_0 d_1 d_1' e_1 d_2 d_2' e_2 \cdots d_n d_n' e_n$, which indicates that the sequence of regions the agent intends to visit is $C_0 C_1 u_1 C_1 C_2 u_2 C_2 \cdots C_{n+1}$.
                That is, the itinerary calls for the agent to travel down the corridor, visiting each room exactly once in the specific order $u_1, u_2, \ldots, u_n$, without backtracking within the corridor.
                For the system to be able to tell that the agent has visited a room, at least one occupancy sensor in each room must be turned on. 
                Also, each of the beam sensors must be turned on so that the system can assure that the agent did not vacillate back-and-forth between cells in the corridor.
                %
                %
                %

\begin{figure}[t]
  \vspace*{-4pt}
  \centering
  \includegraphics[width=\linewidth]{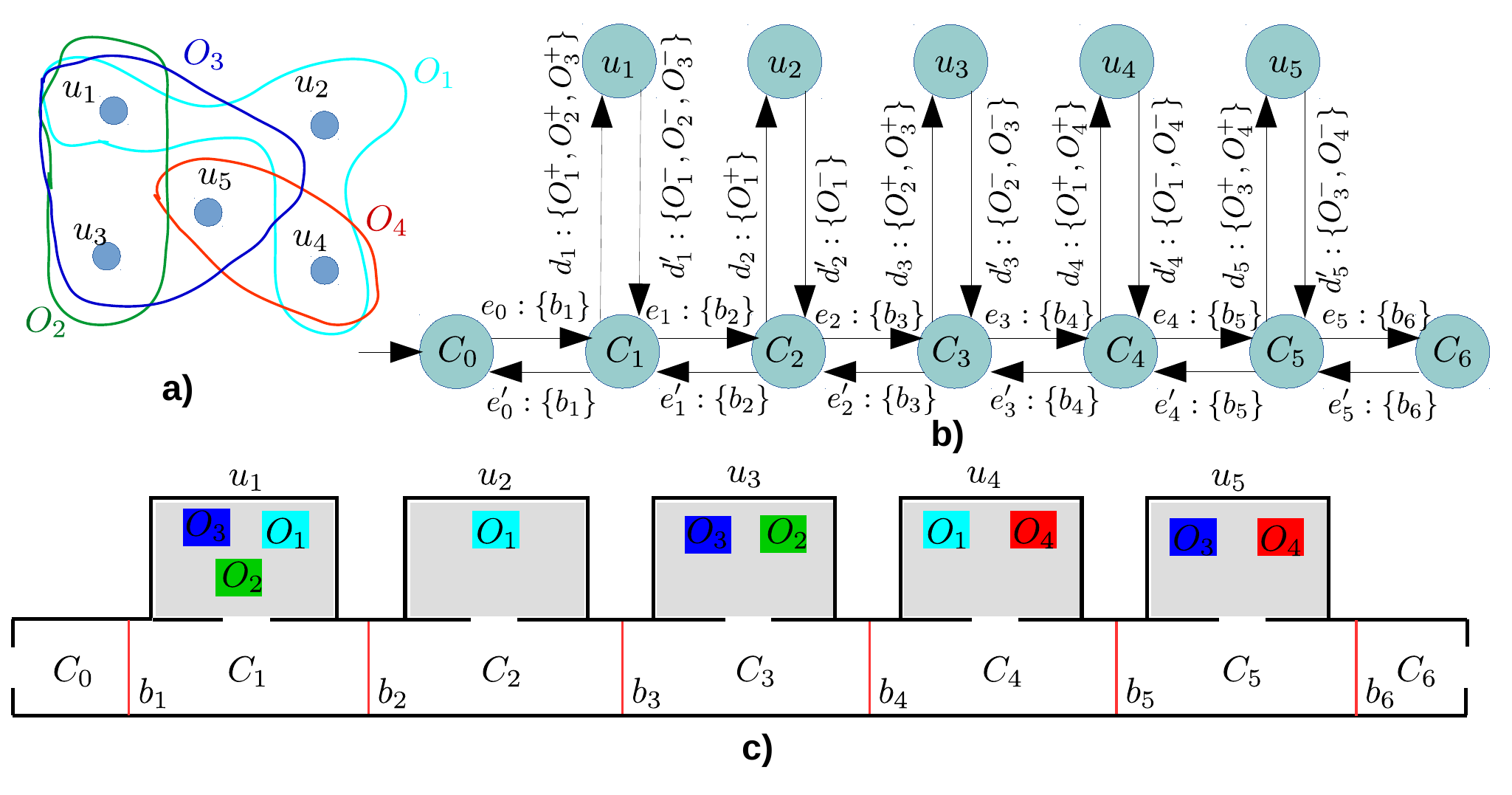}
  \caption{
    \textbf{a)} An instance of the set cover problem.
    \textbf{b)} The instance of the \MSSTVS problem for this set cover instance.
    \textbf{c)} A physical environment representing the \MSSTVS instance.
    }
  \vspace*{-10pt}
  \label{fig:reduction}
\end{figure}

                The construction clearly takes polynomial time, so it remains only to show that the reduction is correct, i.e. that the original SETCOVER-DEC instance has a set cover of size at most $k$ if and only if the constructed \MSSTVSDEC has a valid sensor set of size $k'$.

                 ($\Rightarrow$) Suppose there exists a set cover $\mathbf{N} \subseteq \mathbf{O}$ such that $|\mathbf{N}| \leq k$ and $\sum_{O \in \mathbf{N}} O = U$.
                 In this case, based on our discussion, the sensor selection $M = \mathbf{N} \cup \{b_1, b_2, \ldots, b_{n+1}\}$ is for itinerary $\mathcal{I}$, a certifying sensor selection of size $|M| = |\mathbf{N}|+n+1 \le k+n+1 = k'$.
                 
                 ($\Leftarrow$) Conversely, suppose there exists for $\mathcal{I}$, a certifying sensor selection $M \subseteq S$ for which $|M| \le k'$.  As argued above, because $M$ is certifying, it must contain each of the $n+1$ beam sensors.  Thus, there are at most $k'-(n+1) = k$ occupancy sensors in $M$.  Recall, however, that for this construction, every certifying sensor selection includes at least one occupancy sensor within each room.  Thus, the occupancy sensors in $M$ form a set cover of size at most $k$ for the original SETCOVER-DEC instance.
        \end{proof}

        Finally, the following results follow immediately from Theorem~\ref{thm:nphard} and Lemma~\ref{lem:NP}.
        
        \begin{theorem}
            \MSSTVSDEC is $\NP$-complete.
        \end{theorem} \label{thr:npcomplete}
        
        \begin{corollary} \label{colar:nphard}
          \MSSTVS is $\NP$-hard.
        \end{corollary}
        
        As a result, assuming $\P \neq \NP$, one cannot find a certifying sensor selection with minimum size in polynomial time.
\section{\MSSTVS via Integer Linear Programming}\label{sec:ilp} \pagebudget{1}
In this section, we present an exact solution to \MSSTVS using an Integer Linear Programming
formulation of the problem.  First, we cast the \MSSTVS problem into
mathematical programming form, and then linearize its constraints\gobble{ to form
an integer linear program}.

\begin{table*}
    \caption{Results of our experiment for the environment in Fig.~\ref{fig:department}, which shows the floor map of an environment.}
            \label{tab:caseStudy1}
            \centering    
            \vspace*{-10pt}
\begin{tabular}{ c l  l  l  l}
    \hline
       \textbf{} & \textbf{Itinerary} & \textbf{Description of itinerary} & \textbf{Computed solution} & \textbf{Comp. time (sec)} \\ \hline
        1 & $\walks(\mathcal{G})$ & All location sequences & $\emptyset$ & 2.7 \\ \hline
        2 & $E^* \setminus Walks(\mathcal{G})$ & No location sequence & $\emptyset$ & 2.3 \\ \hline
        3 & $\walks(\mathcal{G}) \setminus\{r \in E^* \mid e_{21} \notin r  \}$ & Do not enter $G$ & $\{o_3\}$ & 2.8 \\ \hline
        4 & $e_1 e_5 e_9 e_{15}$ & $R_1 A B R_2  D $ & \textsc{infeasible} & 2.5 \\ \hline
        5 & $e_1 e_5 e_9 e_{15}+e_1 e_5 e_7 e_{13}$ & $R_1 A B  R_2  D$ or $R_1 A B 
    R_1  C$ & $\{o_1, b_1, b_2, o_2, b_3\}$ & 2.9 \\ \hline
        6 & $e_{12}(e_{11}e_{12}+e_{20}e_{19})^* e_{20} e_{21}$ & $R_1 R_2 (R_1R_2+R_3R_2)^* R_3 G$ & $\{o_1, o_2, o_3, b_1\}$ & 4.1 \\ \hline
    \hline
\end{tabular}
\vspace*{-12pt}
\end{table*}


    
		


	\subsection{Mathematical programming formulation of \MSSTVS}
    For each sensor $s \in S$, we introduce a binary variable $u_s$, which receives value 1 if sensor $s$ is chosen to be turned on, and receives 0 otherwise.
    For each tuple of states $(q, p) \in Q_{\mathcal{P}} \times Q_{\mathcal{P}}$, we introduce a binary variable $a_{q,p}$, which
    receives value 1 if and only if there exist two finite words $r, r' \in E^*$, such that 
    $\delta_{\mathcal{P}}^*(q_0^{\mathcal{P}}, r) = q$,
    $\delta_{\mathcal{P}}^*(q_0^{\mathcal{P}}, r')= p$, and
    and $\beta_{\mathcal{G}}(r, M) = \beta_{\mathcal{G}}(r', M)$.
    For each edge $e \in E$ and event $e \in Y$, we introduce a binary variable $b_{e,y}$ which is assigned a value 1 if and only if the label of $e$ contains $y$ and the sensor that produces $y$ is chosen to be turned on. More precisely, if $y \in \lambda(e)$ and $u_{\eta(y)} = 1$, then $b_{e,y}$ receives value~1, otherwise it receives value 0, where $\eta(y)$ is the sensor that produces event $y$, i.e., $\eta(y)=s$ such that $y \in Y_s$.
    In terms of these variables, an \MSSTVS instance can be expressed as follows.
    
    \begin{mdframed}
\small
	\noindent Minimize:\vspace{-0.5em}
	\begin{equation}  \label{mp:obj}
		\sum_{s \in S} u_s
	\end{equation}
	Subject to:
	        
	\begin{enumerate}[]
	    \item \vspace{-\baselineskip}\begin{equation} \label{mp:init}
			    a_{q_0^{\mathcal{P}},q_0^{\mathcal{P}}} = 1  
		    \end{equation}  	
		\item For each $q \in F_{\mathcal{P}} , p \in Q_{\mathcal{P}} \setminus F_{\mathcal{P}}$,
	        \begin{equation} \label{mp:conflict1}
			    a_{q,p} = 0 
		    \end{equation}  
		\item For each $q \in Q_{\mathcal{P}} \setminus F_{\mathcal{P}}, p \in F_{\mathcal{P}}$,
	        \begin{equation} \label{mp:conflict2}
			    a_{q,p} = 0 
		    \end{equation}  		    
		\item For each $e \in E$ and $y \in Y$ s.t. $y \notin \lambda(e)$,
		    \begin{equation} \label{mp:edgeSensorNoInter}
		        b_{e,y} = 0
		    \end{equation}
        \item For each $e \in E$ and $y \in Y$ s.t. $y \in \lambda(e)$,
		    \begin{equation} \label{mp:edgeSensorInter}
		        b_{e,y} = u_{\eta(y)}
		    \end{equation}	
		\item For each $q, p \in Q_{\mathcal{P}}$ and $e, d \in E$ such that $\delta_{\mathcal{P}}(q, e) \neq \bot$ and $\delta_{\mathcal{P}}(p, d) \neq \bot$,
	        \begin{equation} \label{mp:edgesSameLabel}
			    a_{q,p} = 1  \wedge (\forall y \in Y, b_{e,y}=b_{d,y}) \Rightarrow a_{\delta_{\mathcal{P}}(q,e), \delta_{\mathcal{P}}(p,d)} = 1
		    \end{equation}  		    
		\item For each $q, p \in Q_{\mathcal{P}}$ and $e \in E$ such that $\delta_{\mathcal{P}}(q, e) \neq \bot$,
	        \begin{equation} \label{mp:edgeEpsionQ}
			    a_{q,p}=1 \text{ and } (\forall y \in Y, b_{e,y}=0) \Rightarrow  a_{\delta_{\mathcal{P}}(q,e),p}=1
		    \end{equation}  
		\item For each $q, p \in Q_{\mathcal{P}}$ and $e \in E$ such that $\delta_{\mathcal{P}}(p, e) \neq \bot$,
	        \begin{equation} \label{mp:edgeEpsionP}
			    a_{q,p}=1 \text{ and } (\forall y \in Y, b_{e,y}=0) \Rightarrow  a_{q,\delta_{\mathcal{P}}(p,e)}=1
		    \end{equation}  
	\end{enumerate}
\end{mdframed}%

	The objective~(\ref{mp:obj}) is to minimize the number of sensors turned on. 
	Constraint~(\ref{mp:init}) asserts that there exists a world-observation sequence, (the empty string, $\epsilon$) by which both states of the tuple $(q_0^{\mathcal{P}}, q_0^{\mathcal{P}})$ are reachable from the initial state.
	Constraint Sets~(\ref{mp:conflict1}) and~(\ref{mp:conflict2}) ensure that the sensor selection is certifying.
	Constraint Sets~(\ref{mp:edgeSensorNoInter}) and~(\ref{mp:edgeSensorInter}) encode which sensors affect the label of each edge.
	Constraint Set~(\ref{mp:edgesSameLabel}) asserts that if two states $q$ and $p$ are both reachable by a world-observation sequence, then for any edges $e$ and $d$, if those two edges receive the same world-observation under the chosen sensors, then it means states $\delta_{\mathcal{P}}(q, e)$ and $\delta_{\mathcal{P}}(p, d)$ are also reachable by at least one sequence of world-observations under the chosen sensors.
	In fact, these constraints implement the case shown in Fig.~\ref{fig:mp_edge_constraints}a.
	Similarly, Constraint Sets~(\ref{mp:edgeEpsionQ}) and~(\ref{mp:edgeEpsionP}) implement the case in Fig.~\ref{fig:mp_edge_constraints}b. 

	This formulation would be a $0$--$1$ integer linear programming model if Constraint Sets~(\ref{mp:edgesSameLabel}),~(\ref{mp:edgeEpsionQ}), and~(\ref{mp:edgeEpsionP}) were linear.
	Thus, the next section shows how to linearize them.
	%
	%
	\subsection{Integer linear programming formulation of \MSSTVS}
	To linearize Constraint Set~(\ref{mp:edgesSameLabel}), first we introduce a binary variable $j_{e,d,y}$ for each $e, d \in E$ and $y \in Y$, which receives its value from the following constraints.
	    \begin{mdframed}\small
        \begin{enumerate}[]
            \item For all $e,d \in E$ and all $y \in Y$,
                \begin{eqnarray}
                    \label{lp:b_ed_leq_j}b_{e,y}-b_{d,y} \leq j_{e,d,y},\\
                    \label{lp:b_de_leq_j} b_{d,y}-b_{e,y} \leq j_{e,d,y},\\
                    \label{lp:j_leq_b_ed} j_{e,d,y} \leq b_{e,y}+b_{d,y}, \text{ and} \\
                    \label{lp:j_leq_2_be_bd} j_{e,d,y} \leq 2-b_{e,y}-b_{d,y}.
                \end{eqnarray}
        \end{enumerate}
    \end{mdframed}
    These linear constraints assign value 0 to $j_{e,d,y}$ if $b_{e,y}=b_{d,y}$; otherwise, they assign value 1 to $j_{e,d,y}$.
    Hence, Constraint Set~(\ref{mp:edgesSameLabel}) is replaced by the following linear constraints.
		    \begin{mdframed}\small
        \begin{enumerate}[]
            \item For each $q, p \in Q_{\mathcal{P}}$ and $e, d \in E$ s.t. $\delta_{\mathcal{P}}(q, e) \neq \bot$ and $\delta_{\mathcal{P}}(p, d) \neq \bot$,
                \begin{equation}\label{lp:edgesSameLabel}
                    (1-a_{q,p})+\sum_{y \in Y}(j_{e,d,y})+a_{\delta_{\mathcal{P}}(q,e),\delta_{\mathcal{P}}(p,d)} \geq 1.
                \end{equation} 
        \end{enumerate}
    \end{mdframed}
    
    \noindent We also replace Constraint Set~(\ref{mp:edgeEpsionQ}) by linearized forms:
 	\begin{mdframed}\small
        \begin{enumerate}[]
            \item For each $q, p \in Q_{\mathcal{P}}$ and $e \in E$ s.t. $\delta_{\mathcal{P}}(q, e) \neq \bot$,
                \begin{eqnarray} \label{lp:edgeEpsionQ}
                     (1-a_{q,p})+\sum_{y \in Y}b_{e,y} + a_{\delta_{\mathcal{P}(q,e)},p} \geq 1.
                \end{eqnarray}
        \end{enumerate}
    \end{mdframed}   
    
    \noindent Similarly, we linearize Constraint Set~(\ref{mp:edgeEpsionP}) as follows.
    	\begin{mdframed}\small
        \begin{enumerate}[]
            \item For each $q, p \in Q_{\mathcal{P}}$ and $e \in E$ s.t. $\delta_{\mathcal{P}}(p, e) \neq \bot$,
                \begin{eqnarray} \label{lp:edgeEpsionP}
                     (1-a_{q,p})+\sum_{y \in Y} b_{e,y} + a_{q,\delta_{\mathcal{P}}(p,e)} \geq 1.
                \end{eqnarray}
        \end{enumerate}
    \end{mdframed}   
    
    Now, we have an ILP formulation of \MSSTVS, which can be solved directly by any of the many existing highly-optimized ILP solvers.
    This ILP formulation not only can be used for obtaining solutions to \MSSTVS but also to compute feasible (rather than optimal) solutions for large instances of world graphs for which exact solutions to \MSSTVS cannot be computed in a reasonable amount of time.
    
    \section{Case studies} \label{sec:caseStudies}
    In this section, we present case studies, using the ILP formulation from the previous section to solve some representative instances of \MSSTVS.
    All trials were performed on an Ubuntu 16.04 computer with a 3.6~GHz processor.

    \subsection{Case study 1: Computer Science Department}
    Recall Fig.~\ref{fig:department}, with a map of a small computer science department.
    For this world graph, we executed several instances of the \MSSTVS with different itineraries to verify the algorithm's correctness. 
    Table~\ref{tab:caseStudy1} shows results on those instances.
    The first two instances consider\gobble{ two} extreme itineraries as boundary test cases.
    For the first, the itinerary consists of all walks on the world graph, including the empty string; in the second instance, the itinerary 
    does not have any walks over the world graph.
    In both of these instances, the optimal solution (which has size 0, i.e. no sensor is needed) was correctly found.
    The third scenario considers an itinerary moving any way, other than entering room $G$.
    To certify this itinerary, it suffices to turn on only one of the sensors $o_3$ and $o_4$, both located in room $G$.  Again, our implementation found this solution correctly.
    The fourth itinerary specifies a single sequence in which the agent enters room $A$ from $R_1$, then it enters room $B$ from the right door between $A$ and $B$, then passes through $R_2$ to enter $D$.
    There is no certifying sensor selection for this itinerary because there is another walk, $e_1 e_5 e_7 e_{13}$, that is not within the claimed itinerary but which produces the same sequence of world-observations, 
    for any selection of sensors.
    In contrast, the fifth scenario, whose itinerary includes both the walk from the fourth scenario
    and additionally $e_1 e_5 e_7 e_{13}$, can be solved with no need to turn on sensors $o_3$ and $o_4$.
    The last itinerary specifies all walks in which the agent does not enter any room but only room $G$ at the end after traveling along the corridor between any of regions $R_1$, $R_2$ and $R_3$.
    For this itinerary, it is required to turn on sensors $o_1$, $o_2$ and $b$ to verify that the robot did not enter any of rooms $A$, $B$, $C$, $D$, $F$. It also requires either $o_3$ or $o_4$ be turned on to ensure the robot enters room $G$ at the end.
    Our program took less than 5 seconds to compute a minimal certifying sensor selection for each of these itineraries.
    
    Though small, this case study \altered{suggests} the correctness of our algorithm.
    The next section tests its scalability.
%
    
%
    \subsection{Case study 2: Where eagles soar}
Ornithologists employ cellular-network devices to track migratory patterns of larger birds, allowing new insights to be gleaned~\cite{meyburg2003migration}.
  Fig.~\ref{fig:caspiansea}a shows aggregated tracking data (from~\cite{eagledata}) for journeys made by eagles over a year.
  The surprisingly infrequent flights across open water might lead one to hypothesize that eagles circle the Caspian Sea (the region made visually salient in the figure). To validate this
  hypothesis would require purchasing data roaming capabilities from cellular-network operators across multiple countries in this region. 
  Fig.~\ref{fig:caspiansea}b shows an approach to model the problem of minimizing these
  costs. The map is divided into subregions, each representing a vertex of a world graph. Edges of the world graph are between adjacent hexagons. When whole
  subregions fall substantially within a single country, they have been
  assigned a color representing the potential of purchasing data service for that
  country.  There are 10 colors, representing the sensor set $S$. We model the
  hypothesis of circling the Caspian Sea by an itinerary containing walks that visit III$\cdot$II$\cdot$I$\cdot$III, 
  III$\cdot$I$\cdot$II$\cdot$III, or their extra cyclic permutations.  The DFA describing this itinerary consists of 7 states.

    The observations provided by the cellular network are akin to the occupancy sensors in the previous example, with events triggered when the eagle enters or leaves each hexagonal cell.
    %
    %
    Our decomposition, shown in Fig.~\ref{fig:caspiansea}b, has 36 colored cells. Accordingly, there are 72 events.  An additional 9 cells are uncolored.
    %
    %
    %
    \altered{It took 734.83 seconds for our implementation to form the ILP model and then it took 270.17 seconds to find an optimal solution.}
    In this solution, only 6 out of the 10 sensors (the colors listed in the caption of Fig.~\ref{fig:caspiansea}) were turned on.    
    \altered{Before finding an optimal solution, the program found feasible solutions of sizes 10, 9, and 8 respectively in 83.07, 90.91, and 176.13 seconds.}
    

  %
  
\section{Conclusion}
\vspace{-0.05in}
We have examined the question of selecting the fewest sensors subject to the
requirement that they have adequate distinguishing power to differentiate
motions conforming to an itinerary from those that do not.  This optimization
question fits the resource minimization concern that underlies several useful
applications.  Our formulation of this problem allows for the possibility that
when a sensor is selected, it can provide readings for events in potentially
multiple places.  
To solve the problem, rather than proposing a custom algorithm,
we give an ILP formulation for it, leveraging decades of
optimization on such solvers.  This approach is seen to solve instances of
moderate size\,---\,including a small-scale case study motivated by wildlife
tracking.  

For future work, the steps that have become standard when dealing with
$\NP$-hard problems remain: seeking special-cases that possess some
additional structure making them easier, understanding the problem
using more nuanced parameterization (i.e., fixed parameter tractability
approaches), and custom heuristics and approximation algorithms. 
\textcolor{black}{Also, there might be room to improve the ILP to make it more effective.}

    \begin{figure}[t]
  \centering
  \includegraphics[width=0.475\linewidth]{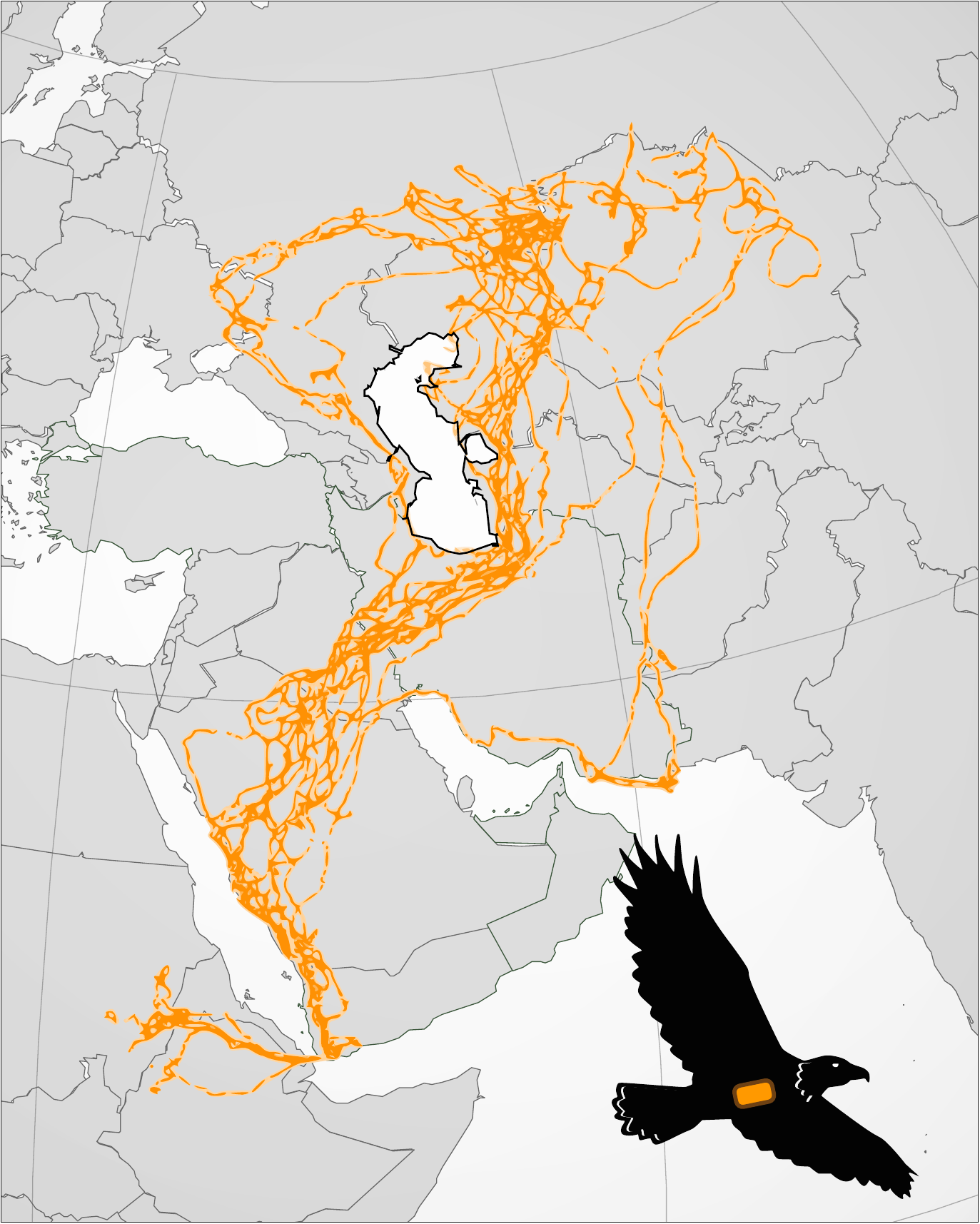}
  \hfill
  \includegraphics[width=0.475\linewidth]{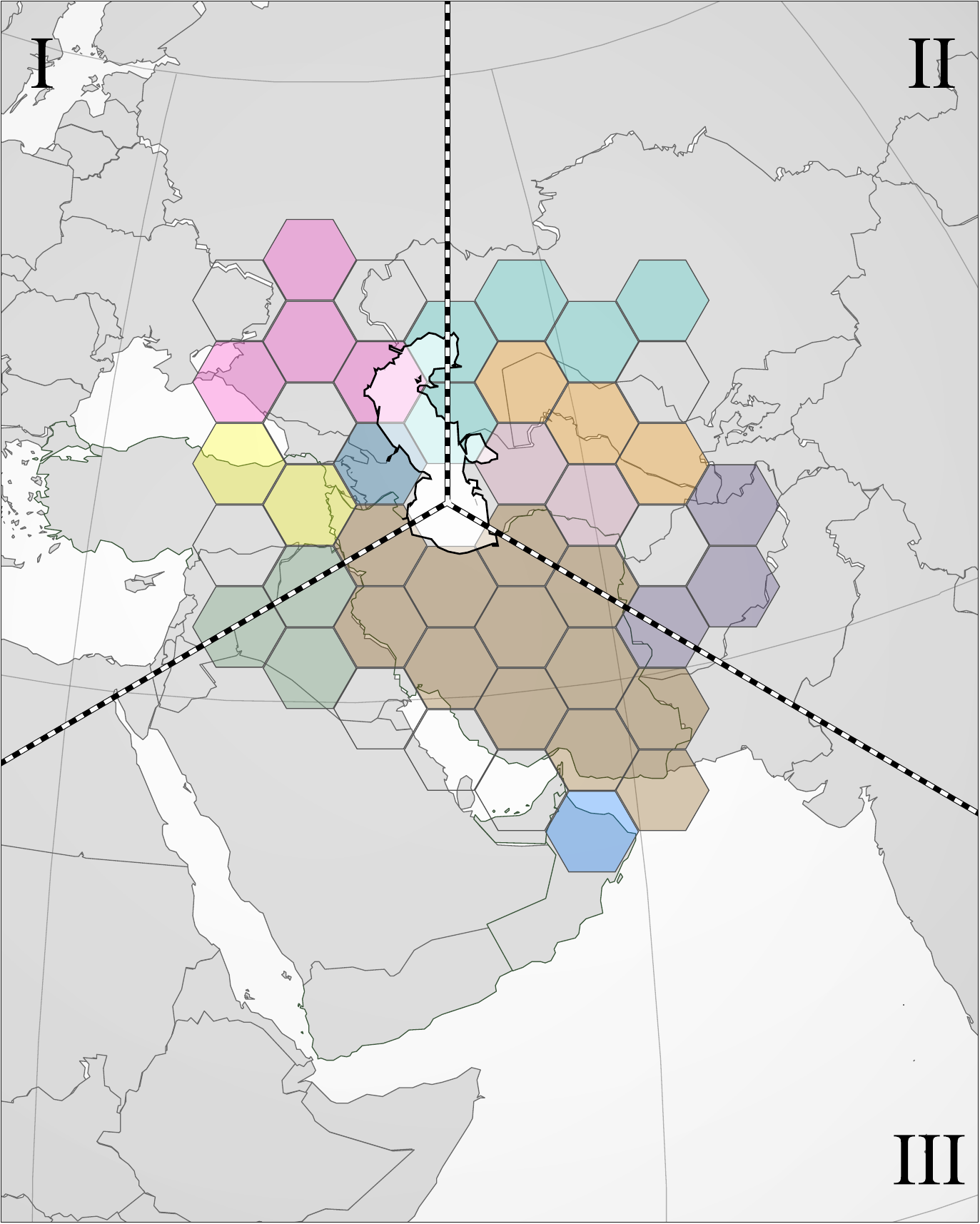}\\
  \vspace*{-4pt}
\textbf{a)} \hspace*{3.7cm} \textbf{b)}\\
  \vspace*{-4pt}
  \caption{\textbf{a)} An aggregate of tracking data, reproduced from~\cite{eagledata}, showing the journeys across parts of North Africa and Central Asia made by eagles over a period of one year.
  Fig. \textbf{b)} The decomposition of the map into subregions.
  Subregions who fall substantially within a single country have been
  assigned a color representing the potential of purchasing data roaming service from that
  country (there are 10 colors, representing sensors set $S$). 
  %
  The optimal sensor selection to certify the hypothesis of circling the
  Caspian Sea has six sensors, consisting of the sets of hexagons of the following colors:
    \includegraphics[scale=0.175]{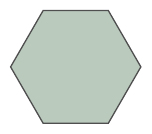},
    \includegraphics[scale=0.175]{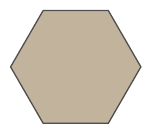},
    \includegraphics[scale=0.175]{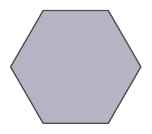},
    \includegraphics[scale=0.175]{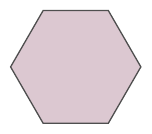},
    \includegraphics[scale=0.175]{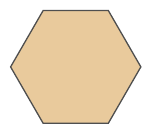}, and
    \includegraphics[scale=0.175]{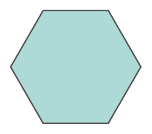}.
  \label{fig:caspiansea}}
  \vspace*{-13pt}
  \end{figure}
\vspace{-0.1in}

\pagebudget{1}
\bibliography{main}

\showtotalpagebudget

\end{document}